\documentclass{l4dc2023}

\title[]{\texttt{FedSysID}: A Federated Approach to Sample-Efficient System Identification}

\newcommand{\fl}{{(i)}}
\newcommand{\E}{\mathbb{E}}
\newtheorem{assume}{Assumption}
\usepackage{algorithm2e}
\usepackage{algorithm}
\usepackage{algpseudocode}
\usepackage[normalem]{ulem} 
\hypersetup{hypertexnames=false}

\usepackage{color}
\definecolor{amethyst}{rgb}{0.6, 0.4, 0.8}

\newcommand{\rev}[1]{{\color{black}#1}}

\author{%
 \Name{Han Wang} \Email{hw2786@columbia.edu}\\
 \addr Columbia University, New York, NY
 \AND
 \Name{Leonardo F. Toso} \Email{lt2879@columbia.edu}\\
 \addr Columbia University, New York, NY%
 \AND
 \Name{James Anderson} \Email{james.andersonc@columbia.edu}\\
 \addr Columbia University, New York, NY
}

\begin{document}

\maketitle

\begin{abstract}
We study the problem of learning a linear system model from the observations of $M$ clients. The catch: Each agent is observing data from a different dynamical system. This work addresses the question of how multiple systems collaboratively learn dynamical models in the presence of heterogeneity. We pose this problem as a federated learning problem and characterize the tension between achievable performance and system heterogeneity. Furthermore, we provide a sample complexity result that obtains a constant factor improvement over the single agent setting. Finally, we describe a meta federated learning algorithm, \texttt{FedSysID}, that leverages existing federated algorithms at the client level. 
\end{abstract}


\section{Introduction}
The system identification problem aims to estimate the parameters of a dynamical system from observed data. The data can take many forms, however, we focus on data sets comprised of a given sequence of input-state trajectories. System identification plays a crucial role in a diverse set of applications including time-series analysis, control theory, robotics, and reinforcement learning. Motivated by trends in big data and the explosion in IoT devices, there has been an increasing interest in finite-sample complexity and non-asymptotic analysis. This line of works can broadly be divided into two categories: i) learning fully observed linear time-invariant (LTI) systems~\citep{sarkar2019near,dean2020sample,simchowitz2018learning};  ii) learning partially observed LTI systems~\citep{oymak2019non,sun2020finite,zheng2020non,simchowitz2019learning}. Despite the mature body of the literature on the finite-sample properties for system identification, it appears that a non-asymptotic convergence rate of $\frac{1}{\sqrt N}$, where $N$ is the number of observed trajectories, is a hard limit. Such dependence on the number of trajectories places a limit on the scalability of  standard techniques for applications where sampling different trajectories is a distinctly difficult and costly task (e.g., large-scale and safety-critical systems).

To circumvent this apparent fundamental sample-efficiency limit, we propose a modified problem setting where multiple similar agents, orchestrated by a central server, collaborate to learn, from a broad data set, a common model that fits all of them well. To achieve this, we leverage the popular federated learning (FL) paradigm~\citep{konevcny2016federated,mcmahan2017communication}. Federated learning (FL) is a machine-learning setting where multiple clients collaborate to train a model under the coordination of a central server. By allowing data to be stored client-side, the FL paradigm has many favorable privacy properties and achieves almost the same performance guarantees as centralized methods. Canonical examples of FL include building treatment outcome models from patients' medical records (without sharing the sensitive data) and creating predictive text models from an ensemble of cell phone users. Compared to traditional distributed optimization, one key challenge and characteristic of FL is heterogeneity;  both in an  \textit{objective} \emph{and} \textit{statistical} sense. In the system identification problem we pose, this heterogeneity manifests through the fact that each client observes a similar (but not identical) dynamical system.

We formulate an offline federated system identification problem. Our problem assumes a central server connected to $M$ clients. The objective is to identify a linear dynamical system model based on observational data held locally by the $M$ clients. The catch - clients observe trajectories from different systems. Specifically, each client observes $N_i$ trajectories generated by a system, with the assumption that all the systems are ``similar'' to each other. Motivating examples of this scenario include modeling multiple autonomous vehicles that operate in similar dynamic environments, based on observations from a fleet of identical vehicles.  Our objective is to answer the following questions: i) \emph{Through collaboration, can each client achieve an improved finite-time convergence guarantee?} ii) \textit{How does the heterogeneity of the observed systems influence the convergence rate?}

\subsection*{Contributions}

\begin{itemize}
  \item \textbf{Problem formulation:} To the best of our knowledge, this is the first paper to use the FL framework for  system identification. We consider  the case where clients  observe similar, but not identical, dynamical systems. Although we demonstrate our solution technique by focusing on the problem of learning fully observed linear dynamical systems from multiple trajectories, our framework can be easily extended to analyze system identification problems with single trajectory, from partially observed linear dynamical systems~\citep{oymak2019non,WanA22}, or sparse dynamical system models~\citep{fattahi2019learning}. The latter can be solved via \texttt{FedADMM}~\citep{wang2022fedadmm}.

  \item \textbf{Improved convergence rate:} Our solution  provides a  convergence rate that is a constant factor better than if a single agent tries to learn a model independently. In particular, the convergence rate is improved by a factor of $\sqrt{M}$, where $M$ is the number of clients, if each client has the same amount of local data. As the number of agents increases, the benefit from collaboration increases. Since the number of clients in FL is usually large, the benefit from collaboration always shows up.

  \item \textbf{Influence of heterogeneity:} We assume that the systems participating in the collaboration are similar but not identical. We provide a theoretical analysis of how the dissimilarity amongst those systems influences the non-asymptotic convergence rate of the proposed technique.

  \item \textbf{Meta Algorithm:} We present \texttt{FedSysID}, a federated learning algorithm to solve the proposed system identification problem. \texttt{FedSysID} can call any FL optimization method as a subroutine. Furthermore, we analyze the convergence property of this algorithm for the specific case of \texttt{FedAvg}~\citep{mcmahan2017communication} and \texttt{FedLin}~\citep{mitra2021linear}. 
\end{itemize}

\paragraph{Notation:}
Given a matrix $A \in \mathbb{R}^{m\times n}$, the Frobenius norm of $A$ is denoted by $||A||_F= \sqrt{Tr(AA^\top)}$. $||A||$ corresponds to the spectral norm of $A$, namely, $||A||=\sigma_{\text{max}}(A)$, where $\sigma_{\text{max}}(A)$ is the largest singular value of $A$. Consider a symmetric matrix $\Sigma$, $\lambda_{\text{min}}(\Sigma)$ and $\lambda_{\text{max}}(\Sigma)$ denote its minimum and maximum eigenvalues, respectively.

\section{Related Work}
\label{sec:related_works}
\textbf{System identification:} Classically, system identification methods provided asymptotic results, see for example~\citep{ljung1998system}. Unfortunately, these results prove consistency using the central limit theorem and law of large numbers type theorems, thus requiring the number of data samples to tend to infinity. Finite-time (i.e., non-asymptotic) analysis of system identification problems is now an active area of research. Subspace methods~\citep{VerV07,van2012subspace}   are often widely used in practice, however, recent work~\citep{JedP20,sarkar2019near,simchowitz2018learning} has shown that the least squares estimator (that we consider) is nearly-optimal in the fully-observed setting using finite-time analysis. System theoretic properties quantifying  controllability~\citep{dean2020sample,TsiP21} and stability~\citep{sarkar2019near} frequently appear in sample complexity bounds. Broadly, identification techniques can be broken down into two categories, depending on whether they require a single trajectory (most suitable for stable systems) or multiple trajectories (suitable for unstable, requires the ability to reset the system). Upper and lower bounds on the sample complexities for such algorithms are concisely reviewed in~\citep{CSM22}. Specifically, various algorithms are shown to have an $\mathcal O(\frac{1}{\sqrt N})$ upper bound on the sample complexity. Our work is motivated by the results of \citep{xin2022identifying}, who introduced the idea of learning system dynamics by leveraging data from a similar system. However, the authors do not formulate the system identification problem in a federated manner. Moreover, they did not establish a connection between their results and the centralized solution where the system has access to the data sets from all the systems.

\noindent \textbf{Federated learning (FL):} FL is a \emph{distributed} machine learning framework with the objective of training a model from multiple client's data~\citep{Mcmahan17,kairouz2021advances}. Motivated by the proliferation of Internet of Things (IoT) devices, FL aims to: i) \emph{provide privacy} by ensuring data does not leave the client device~\citep{truex19}; ii) \emph{reduce communication costs} by reducing the volume of communication between the clients and the server~\citep{kon16}; iii) gracefully handle \emph{client dropout}~\citep{tran2021feddr,wang2022fedadmm}. Typical FL applications comprise of a massive number of clients. Not all clients may be able to fully participate, for example, a cell phone may lose signal and be forced to drop out; iv) handle data heterogeneity~\citep{sahu2018convergence} -- the distributions characterizing client data sets is likely to vary. Although peer-to-peer model exists in  the architecture of FL~\citep{yang2019federated}, we focus on the classical server-client model where a single server periodically aggregates data from its clients. Clients maintain their own data and cannot communicate directly with each other. Numerous FL algorithms have been developed and analyzed. While our algorithm is independent of a specific FL implementation, we highlight results using the classical ~\texttt{FedAvg}~\citep{mcmahan2017communication} and the more recent \texttt{FedLin}~\citep{mitra2021linear} which can handle data heterogeneity.

\section{Centralized System Identification}
\label{sec:centralized_setting}

Given state trajectory data, the centralized system identification problem is to learn a linear time-invariant (LTI) system model that best fits the observations. The LTI model takes the form
\begin{equation}\label{eq:sysLTI}
    x_{t+1} = A x_{t} + B u_{t} + w_{t}, \quad t=0,1,2,\hdots, T-1
\end{equation}
where $x_t \in \mathbb{R}^n$, $u_t \in \mathbb{R}^p$ and $w_t \in \mathbb{R}^n$ are the state, input, and process noise of the system at time $t$, respectively. Moreover, it is  assumed that $\{u_t\}_{t=1}^{\infty}, \{w_t\}_{t=1}^{\infty}$ are independent and identically distributed random variables, namely, $u_{t} \stackrel{\text{i.i.d.}}{\sim} \mathcal{N}\left(0, \sigma_{u}^{2} I_{p}\right)$ and $w_{t}\stackrel{\text{i.i.d.}}{\sim} \mathcal{N}\left(0, \sigma_{w}^{2} I_{n}\right)$. It is further  assumed that the initial state $x_{0} \stackrel{\text{i.i.d.}}{\sim} \mathcal{N}\left(0, \sigma_{x}^{2} I_{n}\right)$.  The state-input trajectory pair $\{x_t, u_t\}$ from a single experiment is referred to as a \textit{rollout}. Multiple rollouts of length $T$ are allowed, with the resulting data set stored as
$\left\{x_{l,t}, u_{l,t}\right\}_{t=0}^{T-1}$, for $l=  1,\hdots  N,$
where $l$ denotes $l$-th rollout and $t$ denotes $t$-th time-step in the rollout. We work from the assumption that the data set is generated by~\eqref{eq:sysLTI} and the objective is to provide estimates $(\hat A, \hat B)$ of best fit. We now revisit the standard least-square estimation procedure.

\subsection{Least-Square Procedure}
Denote $\Theta\triangleq \left[\begin{array}{ll} A& B\end{array}\right] \in \mathbb{R}^{n \times (n+p)}$ where $A$ and $B$ correspond to the ground truth system~\eqref{eq:sysLTI} that generated the data. The augmented state input pair of rollout $l$ at time t, is $ z_{l,t} \triangleq\left[\begin{array}{l}
x_{l,t} \\
u_{l,t}
\end{array}\right] \in \mathbb{R}^{n+p}$. Therefore, the state update $x_{l,t+1}$ can be written as
$$x_{l,t+1} = \Theta z_{l,t} + w_{l,t} \quad  \forall \ 1 \le l \le N \text{ and } 0\le t \le T-1,$$ where $x_{l,t}$ can be expanded recursively as follows
$$
x_{l,t}=G_{t}\left[\begin{array}{c}
u_{l,0} \\
\vdots \\
u_{l,t-1}
\end{array}\right]+F_{t}\left[\begin{array}{c}
{w}_{l,0} \\
\vdots \\
{w}_{l,t-1}
\end{array}\right]+{A^t}{x}_{l,0}, \quad t=1,2,\hdots, T-1
$$
with $
G_{t} \triangleq\left[\begin{array}{llll}
A^{t-1} B & A^{t-2} B & \cdots & B
\end{array}\right]$ and
$F_{t} \triangleq \left[\begin{array}{llll}
A^{t-1} & A^{t-2} & \cdots & I_{n}
\end{array}\right] 
$
for all $t \geq 1$.

\begin{lemma}
Let $u_{l,t}, w_{l,t}$, and $x_{l,0}$ be drawn from i.i.d. Gaussian distributions. For every fixed $t$, $\{z_{l,t}\}_{l=1}^{N}$ are independent and identically distributed as $\mathcal{N}(0,\Sigma_t)$ across rollouts, and for all $t\geq 1$, it holds that
$$
{\Sigma}_{t} \triangleq\left[\begin{array}{cc}
\sigma_{{u}}^{2} {G}_{t} ({G}_{t})^{\top}+\sigma_{{w}}^{2} {F}_{t} ({F}_{t})^{\top}+\sigma_{{i,x}}^{2} {A}^{t} ({A}^{t})^{\top} & 0 \\
0 & \sigma_{{u}}^{2} I_{p}
\end{array}\right] \succ 0.
$$
\label{lemma:covariance_matrix}
\begin{proof}
The proof for Lemma~\ref{lemma:covariance_matrix} is given in the Appendix.
\end{proof}
\end{lemma}
 For each rollout $l$, the data is concatenated as
${X}_{l}=\left[\begin{array}{lll}{x}_{l,T} & \cdots & {x}_{l,1}\end{array}\right] \in \mathbb{R}^{n \times T}, \quad {Z}_l=\left[\begin{array}{lll}{z}_{l,T-1}&\cdots& {z}_{l,0}\end{array}\right] \in \mathbb{R}^{(n+p) \times T},$ and ${W}_l=\left[\begin{array}{llll}{w}_{l,T-1} & \cdots & {w}_{l,0}\end{array}\right] \in \mathbb{R}^{n \times T} .$ We further stack the data  from all the rollouts to obtain the batch matrices ${X}=\left[\begin{array}{lll}{X}_1 & \ldots & {X}_{N}\end{array}\right] \in \mathbb{R}^{n \times N T}, \quad {Z}=\left[\begin{array}{lll}{Z}_{1} & \cdots & {Z}_{N}\end{array}\right] \in \mathbb{R}^{(n+p) \times N T},$ and ${W}=\left[\begin{array}{lll}{W}_{1} & \cdots & {W}_{N}\end{array}\right]\in\mathbb{R}^{n \times N T}.$ Pulling it all together, the state, input noise, and model parameters are related via the system of linear equations,$$X = \Theta Z + W.$$
From this system description, an estimate of system dynamics can be obtained by solving the  unconstrained least-square problem:
\begin{equation}\label{eq:ct_ls}
\hat{\Theta}\triangleq \left[\begin{array}{ll} \hat{A}&\hat{B}\end{array}\right]={\arg\min}_{{\Theta} \in \mathbb{R}^{n \times(n+p)}}\|{X}-{\Theta} {Z}\|_{F}^{2}.
\end{equation}

\vspace*{-1em}
\begin{proposition}\footnote{It is important to remark that \citep{dean2020sample} consider the last time-step of $\Sigma_t$ with $x_0=0$ for the sample complexity characterization, whereas, in our work, we include the information of $\Sigma_t$ from all time steps with $x^{(i)}_{0} \sim \mathcal{N}\left(0, \sigma_{i,x}^{2} I_{n}\right)$ to characterize the non-asymptotic convergence rate of our federated system identification approach.} 
\citep{dean2020sample} Assume we collect data from the linear, time-invariant system initialized at $x_{0}=0$, using inputs $u_{t} \stackrel{\text{i.i.d.}}{\sim} \mathcal{N}\left(0, \sigma_{u}^{2} I_{p}\right)$ for $t=1, \cdots, T$. Suppose  $w_{t} \stackrel{\text{i.i.d.}}{\sim} \mathcal{N}\left(0, \sigma_{w}^{2} I_{n}\right)$ and
$N \geq 8(n+p)+16 \log (4 / \delta)$. Thus, with probability at least $1-\delta$, the least squares estimator using only the final sample of each trajectory satisfies 
$$
\max \left\{ \|\hat{A}-A\|,\|\hat{B}-B\| \right\} \leq \frac{16 \sigma_{w}}{\sqrt{\lambda_{\min }\left(\Sigma_{T-1}\right)}} \sqrt{\frac{(n+2 p) \log (36 / \delta)}{N}}.
$$
\end{proposition}

The above proposition indicates that the estimation error for the least-squares method scales as $\mathcal{O}(N^{-\frac{1}{2}})$. However, the magnitude of $N$ is typically small due to the difficulty and cost of collecting samples, especially for large-scale and safety-critical systems, leading us to a large estimation error. In order to address this issue of insufficient data volume for system identification, we exploit the federated learning (FL) framework, where $N\triangleq \sum_{i=1}^{M} {N_i}$ rollouts are generated from similar, but \emph{not identical} LTI systems, each generating $N_i$ rollouts. Moreover, the number of clients $M$ is usually large, which enables each client to borrow other clients' data indirectly in order to achieve a better sample efficiency, e.g. $\mathcal{O}(\frac{1}{\sqrt{\sum_{i=1}^M N_i}})$, by participating in FL and benefiting from the participation of other clients.

\vspace{-1mm}
\section{Federated System Identification}
\label{sec:federated_setting}

We now pose the federated system identification problem, introduce \texttt{FedSysID} - a federated learning algorithm for system identification, and analyze its performance. Let there be $M$ clients collaborating to solve the system identification problem under the coordination of a central server.  We consider a client-server model of computation, i.e., there is no direct communication between clients, and communication between the central server and client only happens periodically. Each client observes and collects data from a separate dynamical system that is assumed to follow the LTI system dynamics,
\vspace{-1mm}
\begin{align}\label{eq:fed_sys}
    x_{t+1}^{(i)} = A^{(i)} x_{t}^{(i)} + B^{(i)} u_{t}^{(i)} + w_{t}^{(i)}
\end{align}
for $i \in [M]$. To reduce notation, we assume that each system is observed for $T$ time steps. All the results presented carry over to the case where each client collects measurements over differing time horizons. Client $i$ performs $N_i$ rollouts of length $T$ where they excite their systems with an i.i.d. Gaussian input $u_{t}^{(i)} \sim \mathcal{N}\left(0, \sigma_{i,u}^{2} I_{p}\right)$, process noise $w_{t}^{(i)} \sim \mathcal{N}\left(0, \sigma_{i,w}^{2} I_{n}\right)$, and an initial state $x^{(i)}_{0} \sim \mathcal{N}\left(0, \sigma_{i,x}^{2} I_{n}\right)$. The resulting data set for each client is defined as $\left\{x_{l,t}^{(i)}, u_{l,t}^{(i)}\right\}_{t=0}^{T-1}$, for $l=  1,\hdots  N_i$.

The  objective of the federated system identification problem is to find a common estimation of the system matrices $\bar{\Theta}=\left[\begin{array}{ll} \bar{A}& \bar{B}\end{array}\right]$ which provides a small estimation error with respect to each client $i$'s true system $\Theta^{\fl}=\left[\begin{array}{ll} A^{(i)} & B^{(i)}\end{array}\right]$. However, each client $i$ maintains its own data (without sharing it) and can estimate its unknown system dynamics $(A^{(i)},B^{(i)})$ via the least-square procedure~\eqref{eq:ct_ls} described previously. We aim  to answer the question: \textit{Can clients in the FL setting achieve a lower estimation error and a better
sample complexity cost compared to the centralized setting?} To answer this question, we first need to make the following heterogeneity assumption regarding each client's system dynamics.
\begin{assume} (Bounded System Heterogeneity)
There exists a constant $\epsilon$ such that
$$\underset{i,j \in [M]}{\max} \lVert  A^{(i)} -A^{(j)}\rVert \leq \epsilon,\text{ and } 
\underset{i,j \in [M]}{\max}\lVert B^{(i)} -B^{(j)} \rVert\leq \epsilon, \quad \text{holds for all } i,j\in [M].$$
\label{assumption:bnd_sys_heterogeneity}
\end{assume}
\vspace*{-1em}
At the heart of the federated system identification problem is determining when collaboration helps the identification. Intuitively, when clients observe very similar systems, it makes sense to collaborate. For wildly divergent systems, the collaboration will be detrimental. One of our objectives is to provide a more subtle characterization of the effect of system heterogeneity on identification performance. We use $\epsilon$ to characterize this heterogeneity, we will see that it explicitly appears in \texttt{FedSysID}'s performance analysis.

Intuitively, each client can cooperate to take advantage of other clients' data successfully through FL if the system heterogeneity parameter $\epsilon$ is small. Concretely, the objective of the federated system identification problem can be formulated as
\begin{equation}\label{eq:FL}
\bar{\Theta}=\left[\bar{A}, \bar{B}\right]={\arg\min}_{{\Theta} \in \mathbb{R}^{n \times(n+p)}}\frac{1}{M}\sum_{i=1}^M\|{X}^{(i)} -{\Theta} {Z}^{(i)}\|_{F}^{2}.
\end{equation}
where, ${X}^{\fl}=\left[\begin{array}{lll}{X}^{\fl}_1 & \ldots & {X}^{\fl}_{N_{i}}\end{array}\right] \in \mathbb{R}^{n \times N_{i} T},\text{ and 
 }{X}^{\fl}_{l}=\left[\begin{array}{lll}{x}_{l,T}^{\fl} & \cdots & {x}_{l,1}^{\fl}\end{array}\right] \in \mathbb{R}^{n \times T},$ and similarly for $Z^{(i)}$ and $W^{(i)}.$ We characterize the distance between the estimated system dynamics $\bar{\Theta}$, given by Eq~\eqref{eq:FL}, and each client's true system $\Theta^{(i)}$ in the following theorem.
 \begin{theorem}\label{thm:fl_results}\allowdisplaybreaks
 For any fixed $\delta>0$, let $N_{i} \geq$ $\max \left\{8(n+p)+16 \log \frac{2 MT}{\delta}, n \log \frac{(2n+p)T}{\delta}\right\}$ for every $i\in[M]$. Let $\bar{N}\triangleq\sum_{j=1}^{M}N_j$. Then, with probability at least $1-2 \delta$, the estimated dynamics $\bar{\Theta}$ in~\eqref{eq:FL} satisfies
\begin{equation}\begin{aligned}
\label{eq:mainresult}
& \max \left\{\left\|\bar{A}-A^{\fl}\right\|,\left\|\bar{B}-B^{\fl}\right\|\right\}\nonumber \\
& \leq \frac{1}{\sqrt{\bar{N}}}\times
\underbrace{\frac{4C_{0}\sqrt{(n+p)\log\left(\frac{(2n+p)T}{\delta}\right)}
\displaystyle\sum_{t=0}^{T-1}\sqrt{\sum_{j=1}^M \frac{N_j}{\bar{N}}\sigma_{j,w}^2\|\Sigma_t^{(j)}\|}}
{\lambda_{\min}\left(\displaystyle\sum_{j=1}^M \frac{N_j}{\bar{N}}\sum_{t=0}^{T-1}\Sigma_t^{(j)}\right)}}_{C_1:\ \text{error constant}}
+\underbrace{\frac{9\epsilon\displaystyle\sum_{j\neq i}\frac{N_j}{\bar{N}}\sum_{t=0}^{T-1}\|\Sigma_t^{(j)}\|}
{\lambda_{\min}\left(\displaystyle\sum_{j=1}^M \frac{N_j}{\bar{N}}\sum_{t=0}^{T-1}\Sigma_t^{(j)}\right)}}_{\text{heterogeneity bias}}.
\end{aligned}\end{equation}
for all $i\in [M]$, where $C_{0}>0$ and $\Sigma_t^{(i)}$ is the covariance of  $z_{l,t}^{\fl} \triangleq [x_{l,t}^{\fl \top} \; u_{l,t}^{\fl \top}]^\top \in \mathbb{R}^{n+p}$. 
\end{theorem}
\begin{proof}
    All proofs are deferred to the Appendix.
\end{proof}

Examining Theorem~\ref{thm:fl_results}, the estimation errors consist of two terms. The first term matches the estimation error of the centralized setting in which each client has access to all the data from all clients, i.e., to all $\bar N\triangleq \sum_{i=1}^{M}N_i$ rollouts. In particular, when every client has $N$ rollouts, $\bar N=MN$ and the parameter error scales as $\mathcal{O}((MN)^{-1/2})$. The second term is due to client heterogeneity. Theorem~\ref{thm:fl_results} confirms our intuition that each client can improve its estimation performance from $\mathcal{O}({N_i}^{-1/2})$ to $\mathcal{O}((\sum_{i=1}^{M}N_i)^{-1/2})$ by participating in the collaboration, provided that the system heterogeneity $\epsilon$ is small, i.e., when the first term dominates the second term in Theorem~\ref{thm:fl_results}.

The results of Theorem~\ref{thm:fl_results} characterize the optimal solution of the least squares problem~\eqref{eq:FL}. Once this optimal solution $\bar{\Theta}$ is found, this estimation is a good fit for all the similar clients participating in the collaboration. In order to compute $\bar{\Theta}$, we now propose \texttt{FedSysID}, a federated algorithm (see Algorithm~\ref{algFedID} below) that explicitly describes the client-server interaction. Although not the focus of this paper, the algorithm can also handle the case where only a subset of clients participate in each round.

At the beginning of the algorithm, each client has an initial guess $\bar{\Theta}^0$ about its dynamics and chooses $\alpha$ as the step size. At each round $r$, the subset of active users $S_r$ is uniformly sampled from $[M]$.\footnote{ This random sampling step is necessary when a large number of participants are involved as clients may lose connection to the central server due to connectivity or power issues.  }
All clients in $S_r$ execute the ClientUpdate $K_i$-times using their local trajectories before communicating to the server (line 7). By performing multiple local model updates and communicating with the central server for a limited time, the algorithm resolves the core bottleneck of the limited communicating capabilities of clients. Once the local updates are completely performed, each client sends its local model update $\bar{\Theta}_{r+1}^{\fl}$ to the server (line 8). The server  computes an improved global model $\bar{\Theta}_{r+1}$ by averaging the clients' local model updates and sends $\bar{\Theta}_{r+1}$ back to each client (line 11). Full details are presented in Algorithm~\ref{algFedID}.

\begin{algorithm}
\caption{\texttt{FedSysID}} \label{algFedID}
\begin{algorithmic}[1]
\State \textbf{Initialize} the server with $\bar{\Theta}_{0}$ and step size $\alpha$ ;
\State \textbf{Initialize} each client $i\in[M]$ with $\Theta_{0,0}^{\fl}=\bar{\Theta}_{0}$;
\State \textbf{For} each round $r=0, 1, \ldots, R-1$ \textbf{do}
\State \quad \quad uniformly sample $S_r \subseteq \{1,2, \cdots, M\}$ 
\State \quad \quad $\rhd$ {Client side:}
\State \quad \quad \textbf{For} each client $i \in S_r$ \textbf{in parallel do}
\State \quad \quad \quad  $\Theta_{r+1}^{\fl} = \texttt{ClientUpdate}(i, \bar{\Theta}_r, K_i)$
\State \quad \quad \quad send $\Theta_{r+1}^{\fl} $ back to the server
\State \quad \quad \textbf{end for}
\State \quad \quad $\rhd$ {Server side:}
\State \quad \quad  update $\bar{\Theta}_{r+1}=\frac{1}{M}\sum_{i=1}^M \Theta_{r+1}^{\fl}$ and send $\bar{\Theta}_{r+1}$ to each client
\State \textbf{end for}
\State \textbf{Return} $\bar{\Theta}_{R}$
\end{algorithmic}
\end{algorithm}

We describe two different implementations of \texttt{ClientUpdate} in Algorithm~\ref{algFedID}: i) Federated Averaging algorithm (\texttt{FedAvg})~\citep{mcmahan2017communication}; ii) Linear Convergent Federated algorithm (\texttt{FedLin})~\citep{mitra2021linear}.

\begin{itemize}
    \item \texttt{FedAvg:} For each client in $[M]$,  \texttt{ClientUpdate} iterates execute:
      \begin{equation}\label{eq:fedavg}
 \Theta_{r,k}^{\fl} = \Theta_{r,k-1}^{\fl} +\alpha ({X}^{\fl}-\Theta_{r,k-1}^{\fl} {Z}^{\fl})  {Z}^{\fl,\top}, \, \quad k=1,2,\cdots, K_i
    \end{equation}
    with input $\Theta_{r,0}^{\fl}=\Bar{\Theta}_r$ and  output ${\Theta}_{r+1}^{\fl}=\Theta_{r,K_i}^{\fl}$.  From~\citep{pathak2020fedsplit,mitra2021linear}, it is known that both \texttt{FedAvg} cannot converge
    with a constant step size $\alpha$. Thus, with an appropriately chosen stepsize for \texttt{FedAvg},   \texttt{FedSysID} will converge sub-linearly, i.e., $\E\lVert \bar{\Theta}_R -\Theta^{(i)}\rVert \le \mathcal{O}(\frac{1}{KR})$, where the expectation is taken with respect to the sampling scheme, to the optimal solution of problem \eqref{eq:FL}. See Theorem 3 in \citep{woodworth2020minibatch} for details.
    \item \texttt{FedLin:}  Consider the following updating rule of \texttt{ClientUpdate} for all the clients: 
     \begin{equation}\label{eq:fedlin}
     \begin{aligned}
     \Theta_{r,k}^{\fl} &= \Theta_{r,k-1}^{\fl} +\alpha \left[({X}^{\fl}-\Theta_{r,k-1}^{\fl} {Z}^{\fl})  {Z}^{\fl,T} - ({X}^{\fl}-\bar{\Theta}_{r} {Z}^{\fl})  {Z}^{\fl,T} + g_r\right],\\
    g_r &=\frac{1}{M} \sum_{i=1}^M ({X}^{\fl}-\bar{\Theta}_{r} {Z}^{\fl})  {Z}^{\fl,T}, \quad k=1,2,\cdots, K
    \end{aligned}
    \end{equation}
    with input $\Theta_{r,0}^{\fl}=\Bar{\Theta}_r$ and  output ${\Theta}_{r+1}^{\fl}=\Theta_{r,K_i}^{\fl}.$ Note that $g_r$ is the negative gradient of the global function in problem \eqref{eq:FL}. Each client can obtain it by exploiting memory without communication. See~\cite{mitra2021linear} for more details. With proper constant step sizes, \texttt{FedLin} converges linearly to the minimum solution $\bar{\Theta}$, i.e., $\E\lVert \bar{\Theta}_R -\Theta^{(i)}\rVert \le \mathcal{O}(e^{-\beta R})$, where the expectation is taken with respect to the sampling scheme\footnote{Even though, \cite{mitra2021linear} does not mention it, \texttt{FedLin} can also be extended to handle partial participation of clients.} and $\beta$ is the condition number of problem~\eqref{eq:FL}.
\end{itemize}
\allowdisplaybreaks
\begin{corollary}
Frame the hypotheses of Theorem~\ref{thm:fl_results}. For all $R\ge 1,$ the output $\bar{\Theta}_{R}$ given by the \texttt{FedSysID} algorithm satisfies
\allowdisplaybreaks
\begin{equation}\label{eq:corollary_FedAvg}
\allowdisplaybreaks
\max \left\{\E\lVert\bar{A}_{R}-A^{\fl}\rVert,\E\lVert\bar{B}_{R}-B^{\fl}\rVert\right\}  \le \mathcal{O}\left(\frac{1}{KR} + \frac{1}{\sqrt{\sum_{i=1}^M N_i}}\times C_1 +\epsilon \times C_2\right)
\end{equation}
with the \texttt{ClientUpdate} being performed according to the local updating rule of \texttt{FedAvg}. And, 
\allowdisplaybreaks
\begin{equation}\allowdisplaybreaks\label{eq:corollary_FedLin}
\max \left\{\E\lVert\bar{A}_{R}-A^{\fl}\rVert,\E\lVert\bar{B}_{R}-B^{\fl}\rVert\right\}  \le \mathcal{O}\left(e^{-\beta R} + \frac{1}{\sqrt{\sum_{i=1}^M N_i}}\times C_1 +\epsilon \times C_2\right)
\end{equation}
with the \texttt{ClientUpdate} selected from the local updating rule of \texttt{FedLin}. 
\label{corollary:client_update}
\begin{proof}
The proof for this corollary is detailed in the Appendix.
\end{proof}
\end{corollary}

\section{Experimental Results}
\label{sec:experimental_results}

\noindent Numerical results are now shown to illustrate and assess the efficiency  \texttt{FedSysID}.\footnote{All simulations were developed in MATLAB. Codes can be downloaded from \url{https://github.com/jd-anderson/Federated-ID}} We consider clients characterized by dynamical systems with $n=3$ states and $p=2$ inputs, for which the standard deviations of their initial state, input, and process noise are set to be $\sigma_x=\sigma_u=\sigma_w=1$. Following~\citep{xin2022identifying}, we consider the nominal system ($A_0$, $B_0$), and uniformly distributed variables $(\gamma_1^{(i)},\gamma_2^{(i)}) \backsim U(0,\epsilon)$, such that $A^{(i)} = A_0 + \gamma_1^{(i)}V$ and $B^{(i)} =B_0 + \gamma_2^{(i)}U$, where  $V \in \mathbb{R}^{3\times 3}$ and $U \in \mathbb{R}^{3\times 2}$ represent the modification patterns applied to the nominal system described by, 
\begin{align*}
    A_0=\begin{bmatrix}
        0.6 & 0.5 & 0.4\\
        0 &  0.4 & 0.3\\
        0 &   0 & 0.3
    \end{bmatrix}, \quad  V=\begin{bmatrix}
        0 & 0 & 0\\
        0 &  1 & 0\\
        0 &   0 & 1
    \end{bmatrix},\quad B_0=\begin{bmatrix}
        1  & 0.5\\
        0.5 &  1\\
        0.5 & 0.5
    \end{bmatrix},\quad U=\begin{bmatrix}
        1  & 0\\
        0 &  0\\
        0 & 1
    \end{bmatrix}.
\end{align*}

Note that we use the same amount of data for each client, i.e., $N_i = N_j$ for all $i, j \in [M]$, and we set  $T=5$ for each client. Once $\bar{\Theta}$ is computed according to \eqref{eq:FL} via \texttt{FedSysID}, the estimation error is defined as the average of the distances between an estimation $q$ at the r-th global iteration, $\bar{\Theta}_{r}^q$, and the fixed true system, $\Theta^{(1)}$, for different estimations (i.e., estimations made from different data sets). That is, $e_r = \frac{1}{q}\sum_{i=1}^{q}|| \bar{\Theta}_{r}^i - \Theta^{(1)}||$, where $q=25$.

\begin{figure}[ht!]
    \centering
    \includegraphics[width=0.3\textwidth]{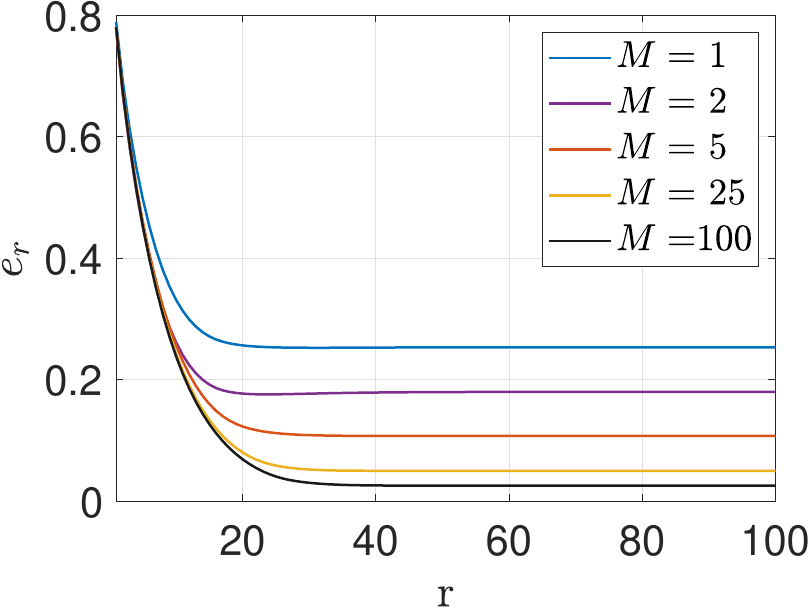}\hfill
    \includegraphics[width=0.3\textwidth]{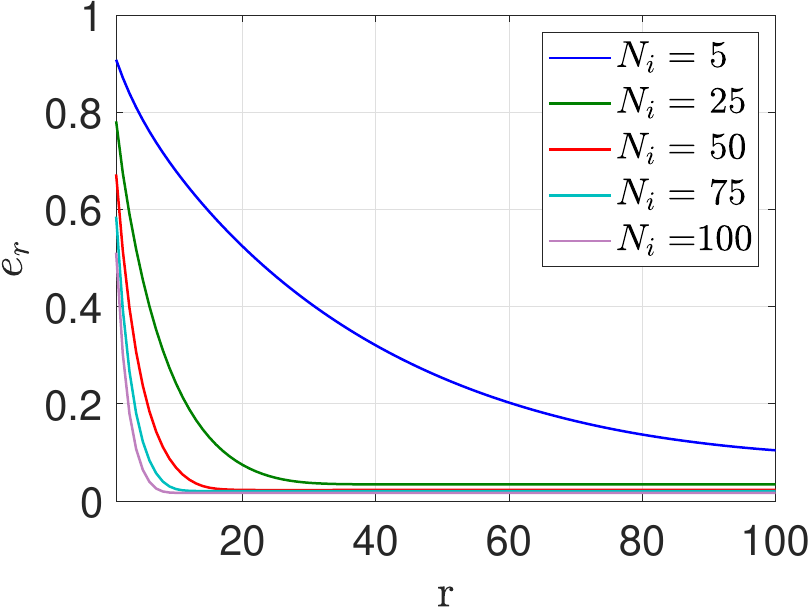}\hfill
    \includegraphics[width=0.32\textwidth]{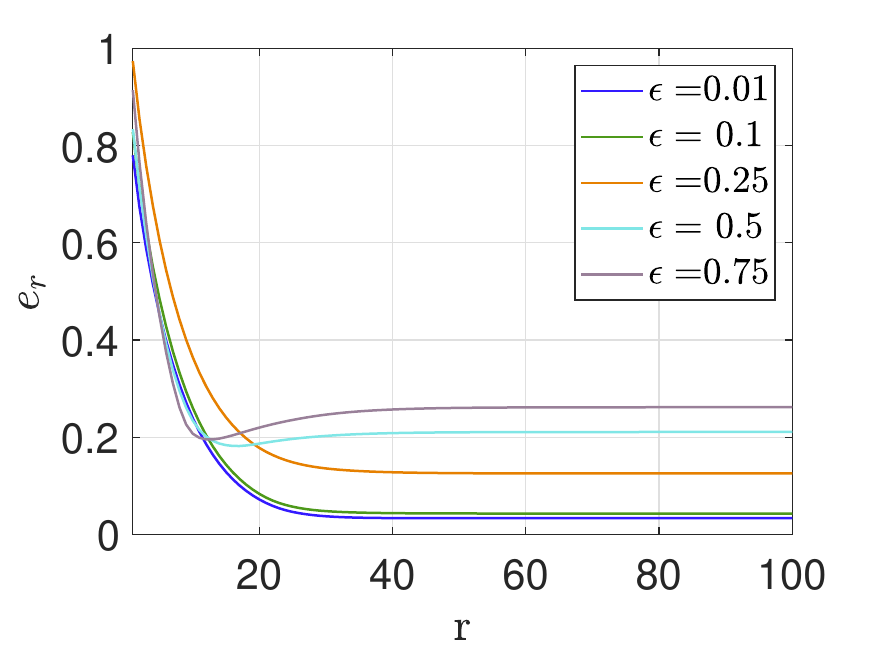}\\
    \makebox[0.3\textwidth]{\rev{(a)}}\hfill\makebox[0.3\textwidth]{\rev{(b)}}\hfill\makebox[0.32\textwidth]{\rev{(c)}}
    \caption{\texttt{FedLin} - Estimation error with respect to the number of global iterations with varying $M$, $N_i$ and $\epsilon$, where each client performs $K_i=10$ local updates with step size $10^{-4}$. (a) $N_i=25$, $\epsilon=0.01$  (b) $M=50$, $\epsilon=0.01$ and (c) $M=50$, $N_i=25$.}
    \label{fig:FedLin}
\end{figure}

Figure \ref{fig:FedLin} illustrates the estimation error with respect to the number of global iterations with varying $M$, $N_i$, and $\epsilon$, when the \texttt{ClientUpdate} is set to the \texttt{FedLin} update. Figure \ref{fig:FedLin} (a) describes the effect of the number of clients $M$ participating in the collaboration on the estimation error, for a fixed number of rollouts $N_i=25$ and a small dissimilarity parameter $\epsilon=0.01$. This figure shows that the estimation error is considerably reduced when the number of clients participating increases. Whilst the estimation error is large for the centralized setting, $M=1$ (i.e., each client learns its dynamics from its own data), it can be reduced by about 10 times when 100 clients collaborate to learn their dynamics by leveraging data from similar systems.  This highlights how the non-asymptotic convergence rate can be improved and scaled by a factor of $\sqrt{M}$ in the low heterogeneity regime.  

Moreover, Figure \ref{fig:FedLin}(b--c) illustrate how the estimation error changes when the number of rollouts $N_i$ and the system's dissimilarity $\epsilon$ increase. Overall, the figures show that the estimation error reduces with a factor of $\sqrt{MN_i}$ and it increases as the heterogeneity of the systems increase. Figure \ref{fig:FedAvg} details a similar analysis in which the same conclusions can be noticed with the \texttt{FedAvg}  \texttt{ClientUpdate} update rule. As expected from Corollary \ref{corollary:client_update},  \texttt{FedSysID} converges faster with \texttt{FedLin} when compared to the \texttt{ClientUpdate} chosen from the local updating rule of \texttt{FedAvg}. 

\begin{figure}[ht!]
    \centering
    \includegraphics[width=0.3\textwidth]{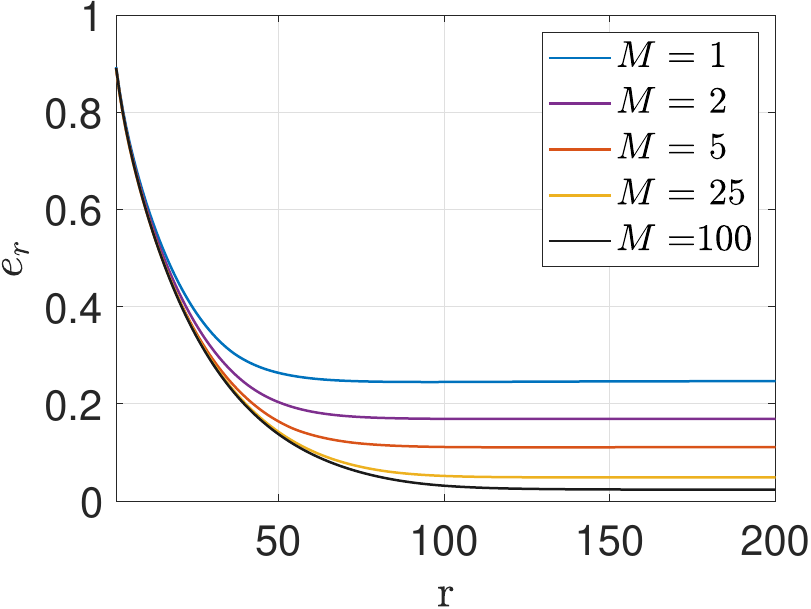}\hfill
    \includegraphics[width=0.3\textwidth]{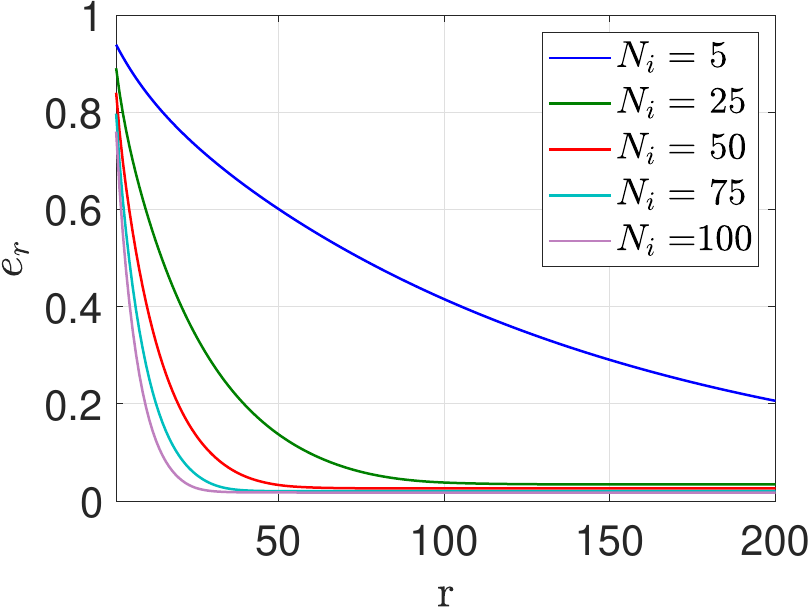}\hfill
    \includegraphics[width=0.32\textwidth]{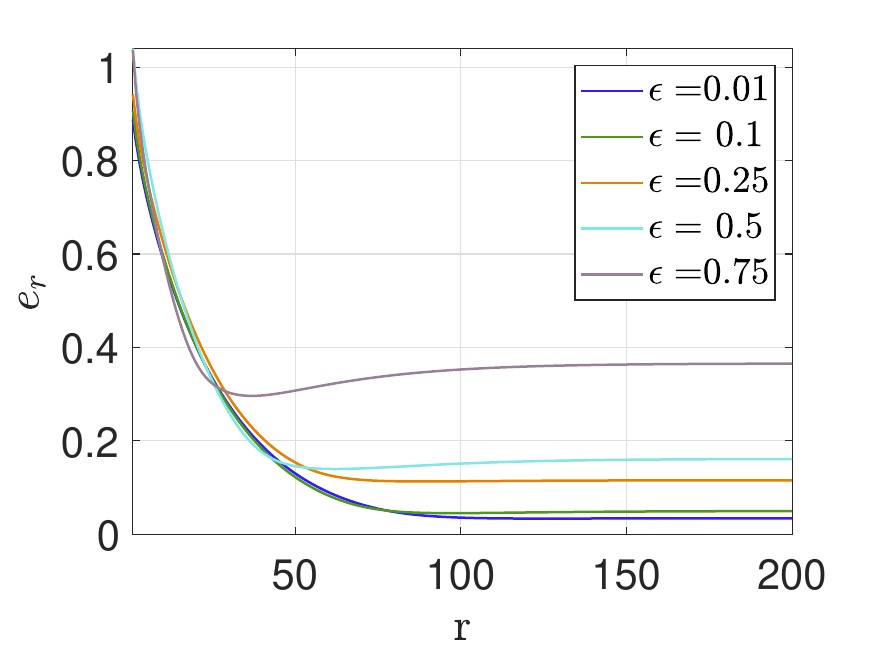}\\
    \makebox[0.3\textwidth]{\rev{(a)}}\hfill\makebox[0.3\textwidth]{\rev{(b)}}\hfill\makebox[0.32\textwidth]{\rev{(c)}}
    \caption{\texttt{FedAvg} - Estimation error with respect to the number of global iterations with varying $M$, $N_i$ and $\epsilon$, where each client performs $K_i=10$ local updates with linear decreasing step size starting at $10^{-4}$. (a) $N_i=25$, $\epsilon=0.01$  (b) $M=50$, $\epsilon=0.01$ and (c) $M=50$, $N_i=25$.}
    \label{fig:FedAvg}
\end{figure}

\vspace*{-2em}

\section{Conclusions}
\label{sec:conclusions}

In this paper, we considered a federated approach for the system identification problem, where a central server orchestrates the collaboration of $M$ clients, with similar dynamics, to identify a common dynamics that suits well all the clients. We demonstrated that a better non-asymptotic convergence rate scaled by $\sqrt{M}$ is achieved when similar systems learn their dynamics by leveraging data from each other in an FL setting. We also provided a \texttt{FedSysID} algorithm that can be used with any FL optimization method as a subroutine. To illustrate the potential of the proposed approach, numerical experiments were demonstrated considering both \texttt{FedLin} and \texttt{FedAvg} updating rules. Future work will explore an extension of the proposed approach for developing an online sample version of \texttt{FedSysID}. We are currently investigating data-driven and system-theoretic heterogeneity metrics.

\acks{The authors would like to thank Aritra Mitra for  valuable and insightful discussions on this work. HW is generously funded by a Wei family fellowship and the Columbia Data Science Institute.  The authors also gratefully acknowledge funding from the DoE under grant DE-SC0022234 and NSF  awards 2144634 \& 2231350}.

\bibliography{reference}

\newpage

\section{Appendix}

In this appendix, we provide the proofs of Lemma \ref{lemma:covariance_matrix}, Theorem \ref{thm:fl_results}, and Corollary \ref{corollary:client_update}. We also include numerical results that confirm the better convergence rate of \texttt{FedLin} when compared to \texttt{FedAvg}. 

\subsection{Proof of Lemma \ref{lemma:covariance_matrix}}

Note that the positive definiteness of $\Sigma_t$, for all $t\geq 1$, is immediately verified due to the presence of an identity matrix $I_n$ in the definition of $F_t$. Therefore, even for an uncontrollable system, i.e., where $G_t(G_t)^\top$ is rank deficient, the term $F_t(F_t)^\top$ ensures the full rank of $\sigma_{{u}}^{2} {G}_{t} ({G}_{t})^{\top}+\sigma_{{w}}^{2} {F}_{t} ({F}_{t})^{\top}+\sigma_{{i,x}}^{2} {A}^{t} ({A}^{t})^{\top}$. Thus, $\Sigma_t$ is positive definite and $\lambda_{\min}(\Sigma_t)>0$ for all $t\geq 1$.

\subsection{Auxiliary results}
The main proofs will make use of the subsequent results.

We can write system $j$'s dynamics as a function of system $i$ as follows:
\begin{align}\label{eq:auxiliary_system}
    x_{t+1}^{(j)}=\left(A^{(i)}+\epsilon^{A}_{j,i}\right) x_{t}^{(j)}+\left(B^{(i)}+\epsilon^{B}_{j,i}\right) u_{t}^{(j)}+w_{t}^{(j)},
\end{align}
where $\epsilon^{A}_{j,i}=A^{(j)}-A^{(i)}$, and $\epsilon^{B}_{j,i}=B^{(j)}-B^{(i)}$, for all $i \neq j \in [M]$. By assuming $u_{t}^{(j)}  \stackrel{\text{i.i.d.}}{\sim} \mathcal{N}\left(0, \sigma_{j,u}^{2} I_{p}\right)$, $w_{t}^{(j)}  \stackrel{\text{i.i.d.}}{\sim} \mathcal{N}\left(0, \sigma_{j,w}^{2} I_{n}\right)$, and $x_{0}^{(j)} \stackrel{\text{i.i.d.}}{\sim} \mathcal{N}\left(0, \sigma_{j,x}^{2}\right)$,  \eqref{eq:auxiliary_system} can be rewritten as, 

\begin{align}
\label{eq:augmentedsys}
 X = \Theta^{(i)} Z +W +\Delta^{(i)}
\end{align}
where, 
$$
X=\left[\begin{array}{ll}X^{(i)},\ldots, X^{(M)}\end{array}\right] \in \mathbb{R}^{n \times MN_{i} T},\quad Z=\left[\begin{array}{ll}Z^{(i)},\ldots, Z^{(M)}\end{array}\right] \in \mathbb{R}^{(n+p)\times MN_{i} T},
$$
$$
W=\left[\begin{array}{ll}W^{(i)},\ldots, W^{(M)}\end{array}\right] \in \mathbb{R}^{n \times MN_{i} T}, $$
and 
$$\Delta^{(i)}=\left[\begin{array}{lllllll} \Delta^{(i,1)}& \cdots & \Delta^{(i,i-1)}& \mathbf{0} & \Delta^{(i, i+1)} & \cdots & \Delta^{(i,M)}\end{array}\right],
$$
with $\Delta^{(i,j)}=\epsilon^{\Theta}_{j,i}Z^{(j)} \in \mathbb{R}^{n \times N_{i} T}$ representing the error caused by using system $i's$ matrices $\Theta^{(i)}=[A^{(i)} \;\ B^{(i)}]$ to describe system $j's$ trajectories, where $\epsilon^{\Theta}_{j,i}=\left[\begin{array}{ll} \epsilon^{A}_{j,i} &\epsilon^B_{j,i}\end{array}\right] \in \mathbb{R}^{n \times(n+p)}$, and $\Delta^{(i,i)}=\mathbf{0}=\left[\begin{array}{llll}0 &0&\cdots&0\end{array}\right]$ $\in \mathbb{R}^{n \times N_{i} T}$. Using \eqref{eq:FL} to find $\Theta^{(i)}$ in \eqref{eq:augmentedsys}, we can write the  estimation error as
\begin{align}\label{eq:apendix_estimation_error}
\bar{\Theta}-\Theta^{(i)}=WZ^\top\left(Z Z^\top\right)^{-1}+\Delta^{(i)}Z^\top\left(ZZ^\top\right)^{-1}.
\end{align}

\begin{lemma}\label{lemma:vershynin2010}
 Let $v_{i} \stackrel{\text{i.i.d.}}{\sim} \mathcal{N}\left(0, I_{n+p}\right), i=1, \ldots, N$. For any fixed $\delta>0$, with probability at least $1-\delta$, it holds that,
$$
\begin{aligned}
&\sqrt{\lambda_{\min }\left(\sum_{i=1}^{M} v_{i} v_{i}^{*}\right)} \geq \sqrt{M}-\sqrt{n+p}-\sqrt{2 \log \frac{2}{\delta}}, \\
&\sqrt{\lambda_{\max }\left(\sum_{i=1}^{M} v_{i} v_{i}^{*}\right)} \leq \sqrt{M}+\sqrt{n+p}+\sqrt{2 \log \frac{2}{\delta}}.
\end{aligned}
$$
\end{lemma}
\begin{proof}
 See Corollary 5.35 of \cite{vershynin2010introduction} for details.
\end{proof}

\begin{proposition}\label{prop:DeltaZ_bound}
For any fixed $\delta>0$, let $N_{i} \geq 8(n+$ $p)+16 \log \frac{2M T}{\delta}$. It holds, with probability at least $1-\delta$, that,
$$
\begin{gathered}
{Z}^{(i)}{Z}^{(i){\top}} \succeq \frac{N_{i}}{4} \sum_{t=0}^{T-1} {\Sigma}_{t}^{(i)} \\
\left\|\Delta^{(i)}Z^{\top}\right\| \leq \sum_{\substack{j=1\\j\neq i}}^{M} \frac{9 N_{j}}{4} \left\|\epsilon^\Theta_{j,i}\right\|\sum_{t=0}^{T-1} \left\|\Sigma_{t}^{(j)}\right\|.
\end{gathered}
$$

\begin{proof}
First, we can write, 

\begin{align}\label{eq:ZZ}
  {Z}^{(i)} {Z}^{(i)\top}=\sum_{l=1}^{N_{i}} Z_l^{(i)} Z_l^{(i)\top}
=\sum_{l=1}^{N_{i}} \sum_{t=0}^{T-1} {z}_{l,t}^{(i)} {z}_{l,t}^{(i)\top}=\sum_{t=0}^{T-1} \sum_{l=1}^{N_{i}} {z}_{l,t}^{(i)} {z}_{l,t}^{(i)\top}  
\end{align}
thus, for any fixed $l,t$, and $i$, we define ${\chi}_{l,t}^{(i)}=(\Sigma_{t}^{(i)})^{-\frac{1}{2}} z_{l,t}^{(i)}$. Note that ${\chi}_{l,t}^{i} \stackrel{\text{i.i.d.}}{\sim} \mathcal{N}\left(0, I_{n+p}\right)$, for all $l \in$ $\left\{1,2, \ldots, N_{i}\right\}$. Hence,  \eqref{eq:ZZ} can be rewritten as,
$$\sum_{t=0}^{T-1} \sum_{l=1}^{N_{i}}(\Sigma_{t}^{(i)})^{\frac{1}{2}} {\chi}_{l,t}^{(i)} {\chi}_{l,t}^{(i)\top}(\Sigma_{t}^{(i)})^{\frac{1}{2}} =\sum_{t=0}^{T-1} (\Sigma_{t}^{(i)})^{\frac{1}{2}}\left(\sum_{l=1}^{N_{i}} {\chi}_{l,t}^{(i)} {\chi}_{l,t}^{(i)\top}\right) (\Sigma_{t}^{(i)})^{\frac{1}{2}}$$ 
where,

\allowdisplaybreaks
\begin{align*}\allowdisplaybreaks \left\|\Delta^{(i)} Z^{\top}\right\| &=\left\|\sum_{\substack{j=1\\j\neq i}}^{M} \Delta^{(i,j)}Z^{(j)\top}\right\|=\left\|\sum_{\substack{j=1\\j\neq i}}^{M} \sum_{l=1}^{N_{j}} \sum_{t=0}^{T-1} \epsilon^\Theta_{j,i} {z}_{l,t}^{(j)} {z}_{l,t}^{(j)\top}\right\| \\ &\leq \sum_{\substack{j=1\\j\neq i}}^{M} \left\|\sum_{t=0}^{T-1} \sum_{l=1}^{N_{j}} \epsilon^\Theta_{j,i} {z}_{l,t}^{(j)} {z}_{l,t}^{(j)\top}\right\| \\ & \leq \sum_{\substack{j=1\\j\neq i}}^{M} \sum_{t=0}^{T-1} \left\|\epsilon^\Theta_{j,i}\right\|\left\|\sum_{l=1}^{N_{j}} z_{l,t}^{(j)} z_{l,t}^{(j) \top}\right\| \\ &=\sum_{\substack{j=1\\j\neq i}}^{M} \sum_{t=0}^{T-1} \left\|\epsilon^\Theta_{j,i}\right\|\left\|({\Sigma_{t}^{(j)}})^{\frac{1}{2}}\left(\sum_{l=1}^{N_{j}} \chi_{l,t}^{(j)} {\chi}_{l,t}^{(j)\top}\right) ({\Sigma_{t}^{(j)}})^{\frac{1}{2}}\right\| \\ & \leq \sum_{\substack{j=1\\j\neq i}}^{M}  \sum_{t=0}^{T-1} \left\|\epsilon^\Theta_{j,i}\right\|\left\|{\Sigma}_{t}^{(j)}\right\|\left\|\sum_{l=1}^{N_{j}} \chi_{l,t}^{j} \chi_{l,t}^{(j)\top}\right\|. 
\end{align*}

Therefore, by fixing $\delta>0$ and $t$, and applying Lemma \ref{lemma:vershynin2010}, we obtain, with probability at least $1-\frac{\delta}{MT}$ the following expressions, 
$$
\sqrt{\lambda_{\min }\left(\sum_{l=1}^{N_{j}} \chi_{l,t}^{(j)} \chi_{l,t}^{(j)\top}\right)} \geq \sqrt{N_{j}}-\sqrt{n+p}-\sqrt{2 \log \frac{2MT}{\delta}},
$$
$$
\sqrt{\lambda_{\max }\left(\sum_{l=1}^{N_{j}} \chi_{l,t}^{(j)} \chi_{l,t}^{(j)\top}\right)} \leq \sqrt{N_{j}}+\sqrt{n+p}+\sqrt{2 \log \frac{2M T}{\delta}}.
$$
Further, we have
$$
\begin{aligned}
 \frac{1}{2} \sqrt{N_{j}} \geq \sqrt{n+p}+\sqrt{2 \log \frac{2 MT}{\delta}},
\end{aligned}
$$
which can be rearranged  to give 
$$
\begin{aligned}
\frac{N_{j}}{4} \geq\left(\sqrt{n+p}+\sqrt{2 \log \frac{2M T}{\delta}}\right)^{2}
\end{aligned},
$$
thus, by using the inequality $2\left(a^{2}+b^{2}\right) \geq(a+b)^{2}$, we can write
$$
2\left(n+p+2 \log \frac{2M T}{\delta}\right) \geq\left(\sqrt{n+p}+\sqrt{2 \log \frac{2 MT}{\delta}}\right)^{2}.
$$

Therefore, as long as $N_{j} \geq 8(n+p)+16 \log \frac{2M T}{\delta}$, the following inequalities hold, with probability at least $1-\frac{\delta}{MT}$,
\begin{equation*}
\begin{aligned}
\sqrt{\lambda_{\min }\left(\sum_{l=1}^{N_{j}} \chi_{l,t}^{(j)} \chi_{l,t}^{(j)\top}\right)} 
&\geq \frac{1}{2} \sqrt{N_{j}}+\frac{1}{2} \sqrt{N_{j}}
-\sqrt{n+p}-\sqrt{2 \log \frac{2 MT}{\delta}}\\
&\geq \frac{1}{2} \sqrt{N_{j}},
\end{aligned}
\end{equation*}
\begin{equation}
\begin{aligned}
\sqrt{\lambda_{\max }\left(\sum_{l=1}^{N_{j}} \chi_{l,t}^{(j)} \chi_{l,t}^{(j)\top}\right)} &\leq \sqrt{N_{j}}+\sqrt{n+p}+\sqrt{2 \log \frac{2 MT}{\delta}}\\
&\leq \frac{3}{2} \sqrt{N_{j}},
\end{aligned}
\end{equation}
then, it follows that, with probability at least $1-\frac{\delta}{T}$, we can write, 
$$
 (\Sigma_{t}^{(i)})^{\frac{1}{2}}\left(\sum_{l=1}^{N_{i}} {\chi}_{l,t}^{(i)} {\chi}_{l,t}^{(i)\top}\right) (\Sigma_{t}^{(i)})^{\frac{1}{2}} \succeq  \frac{N_{i}}{4} \Sigma_{t}^{(i)},
$$
and, 
$$
\left\|\epsilon^\Theta_{j,i}\right\|\left\|{\Sigma}_{t}^{(j)}\right\|\left\|\sum_{l=1}^{N_{j}} \chi_{l,t}^{(j)} \chi_{l,t}^{(j)\top}\right\| \leq \frac{9 N_{j}}{4}\left\|\epsilon^\Theta_{j,i}\right\|\left\|\Sigma_{t}^{(j)}\right\|.
$$

Therefore, the proof is completed by applying the union bound for all $0 \leq t \leq T-1$ and $j \in [M]$.
\end{proof}
\end{proposition}

\rev{\begin{lemma}\label{lemma:matrix_bernstein}
Let $\{(f_k,g_k)\}_{k=1}^{K}$ be independent pairs of zero mean Gaussian vectors, where $f_k\in\mathbb{R}^{m}$ and $g_k\in\mathbb{R}^{d}$ are independent within each pair, with covariance matrices $\Sigma_{f,k}$ and $\Sigma_{g,k}$, respectively. There is a universal numerical constant $c>0$ such that, for any $\delta\in(0,1)$, with probability at least $1-\delta$,
$$
\left\|\sum_{k=1}^{K}f_kg_k^\top\right\|
\leq c\left(\sqrt{v\log\frac{m+d}{\delta}}+L\log\frac{m+d}{\delta}\right),
$$
where
$$
v=\max\left\{\left\|\sum_{k=1}^{K}\operatorname{Tr}(\Sigma_{g,k})\Sigma_{f,k}\right\|,
\left\|\sum_{k=1}^{K}\operatorname{Tr}(\Sigma_{f,k})\Sigma_{g,k}\right\|\right\} \text{ where }
L=\max_k\sqrt{\operatorname{Tr}(\Sigma_{f,k})\operatorname{Tr}(\Sigma_{g,k})}.
$$
\end{lemma}}

We note that under the sample-size condition used below, the second term is absorbed into the first up to a universal constant. This result follows by applying the subexponential matrix Bernstein inequality to the self-adjoint matrices
$$
\begin{bmatrix}
    0_{m\times m}  & f_k g^\top_k\\
    g_k f^\top_k   & 0_{d\times d}
\end{bmatrix}, \text{ for all } k=1,\ldots,K 
$$
We refer the reader to \cite[Theorem~6.2]{tropp2012user}.

\begin{proposition}\label{prop:WZ_bound}
For any fixed $\delta>0$, let $N_{i} \geq n\log \frac{(2n+p)T}{\delta}$ for every $i \in [M]$. It holds, with probability at least $1-\delta$, that
\rev{\[
\left\|W Z^{\top}\right\| \leq C_{0}\sqrt{(n+p)\log\frac{(2n+p)T}{\delta}}
\sum_{t=0}^{T-1}\sqrt{\sum_{i=1}^{M}N_i\sigma_{i,w}^2\|\Sigma_t^{(i)}\|}.
\]}

\rev{\begin{proof}
For a fixed time $t$, the pairs $\{(w_{l,t}^{(i)},z_{l,t}^{(i)})\}_{i,l}$ are independent across clients and rollouts, and $w_{l,t}^{(i)}$ is independent of $z_{l,t}^{(i)}$. Applying Lemma~\ref{lemma:matrix_bernstein} jointly over $(i,l)$, with $\Sigma_{f,(i,l)}=\sigma_{i,w}^2I_n$ and $\Sigma_{g,(i,l)}=\Sigma_t^{(i)}$, yields, with probability at least $1-\delta/T$,
$$
\left\|\sum_{i=1}^{M}\sum_{l=1}^{N_i}w_{l,t}^{(i)}z_{l,t}^{(i)\top}\right\|
\leq C_{0}\sqrt{(n+p)\log\frac{(2n+p)T}{\delta}}\sqrt{\sum_{i=1}^{M}N_i\sigma_{i,w}^2\|\Sigma_t^{(i)}\|}.$$
The vectors associated with different time indices need not be independent. We therefore sum the preceding bound over $t$ and apply a union bound only across the $T$ time indices, which proves the claim without imposing independence across time within a rollout.
\end{proof}}
\end{proposition}

\subsection{Proof of Theorem \ref{thm:fl_results}}

To prove  Theorem \ref{thm:fl_results} we start with, 
\begin{align*}
    \max \left\{\left\|\bar{A}-A^{\fl}\right\|,\left\|\bar{B}-B^{\fl}\right\|\right\} \leq \left\|\bar{\Theta}-\Theta^{\fl}\right\|.
\end{align*}
By using \eqref{eq:apendix_estimation_error} and the triangle inequality, we obtain
\begin{align*}
\left\|\bar{\Theta}-\Theta^{(i)}\right\|\leq \left\|WZ^\top\right\|\left\|\left(Z Z^\top\right)^{-1}\right\|+\left\|\Delta^{(i)}Z^\top\right\| \left\|\left(Z Z^\top\right)^{-1}\right\|.
\end{align*}

Therefore, from Proposition \ref{prop:DeltaZ_bound}, we obtain, with probability $1-\delta$, 
\begin{equation}\label{eq:ZZ_bound}
    \left\|\left(ZZ^{\top}\right)^{-1}\right\| \leq \frac{4}{\lambda_{\min }\left(\sum_{i=1}^M N_{i} \sum_{t=0}^{T-1} {\Sigma}^{(i)}_t\right)}
\end{equation}
as long as $N_{i} \geq$ $\max \left\{8(n+p)+16 \log \frac{2 MT}{\delta}, n \log \frac{(2n+p)T}{\delta}\right\}$, for all $i\in [M]$. \rev{Letting $\bar{N}=\sum_{j=1}^{M}N_j$, Proposition~\ref{prop:DeltaZ_bound}, Proposition~\ref{prop:WZ_bound}, and~\eqref{eq:ZZ_bound} imply
\begin{align*}
\left\|\bar{\Theta}-\Theta^{(i)}\right\|
&\leq \frac{1}{\sqrt{\bar{N}}}
\frac{4C_{0}\sqrt{(n+p)\log\left(\frac{(2n+p)T}{\delta}\right)}
\displaystyle\sum_{t=0}^{T-1}\sqrt{\sum_{j=1}^M \frac{N_j}{\bar{N}}\sigma_{j,w}^2\|\Sigma_t^{(j)}\|}}
{\lambda_{\min}\left(\displaystyle\sum_{j=1}^M \frac{N_j}{\bar{N}}\sum_{t=0}^{T-1}\Sigma_t^{(j)}\right)}\\
&\quad +\frac{9\displaystyle\sum_{j\neq i}\frac{N_j}{\bar{N}}\|\epsilon_{j,i}^{\Theta}\|\sum_{t=0}^{T-1}\|\Sigma_t^{(j)}\|}
{\lambda_{\min}\!\left(\displaystyle\sum_{j=1}^M \frac{N_j}{\bar{N}}\sum_{t=0}^{T-1}\Sigma_t^{(j)}\right)}.
\end{align*}
Applying Assumption~\ref{assumption:bnd_sys_heterogeneity}, which gives $\|\epsilon_{j,i}^{\Theta}\|\leq\epsilon$, completes the proof.}

\subsection{Proof of Corollary \ref{corollary:client_update}}

Corollary \ref{corollary:client_update} is a direct result of Theorem \ref{thm:fl_results} when the \texttt{ClientUpdate} in \texttt{FedSysID} is selected from the updating rule of \texttt{FedAvg} and \texttt{FedLin}. To prove this corollary, we start with, 
\begin{align*}
    \max \left\{\left\|\bar{A}_R-A^{\fl}\right\|,\left\|\bar{B}_R-B^{\fl}\right\|\right\} \leq \left\|\bar{\Theta}_R-\Theta^{\fl}\right\|,
\end{align*}
where, 
\begin{align*}
  \left\|\bar{\Theta}_R-\Theta^{\fl}\right\| \leq \underbrace{\left\|\bar{\Theta}_R-\bar{\Theta}\right\|}_{(I)} + \underbrace{\left\|\bar{\Theta} -\Theta^{\fl}\right\|}_{(II)},
\end{align*}
by using the triangle inequality. Thus, the results presented in \eqref{eq:corollary_FedAvg} and \eqref{eq:corollary_FedLin} are immediately checked by combining the proof of Theorem \ref{thm:fl_results}, for the term $(II)$, with the corresponding sub-linear and linear non-asymptotic convergence rates of \texttt{FedAvg} and \texttt{FedLin}, detailed in \citep{woodworth2020minibatch, mitra2021linear}, for the term $(I)$. Note that $\beta$ is the condition number of problem \eqref{eq:FL}, which is a fixed constant.

\subsection{Additional Numerical Results}

The results detailed in Corollary \ref{corollary:client_update} highlight the sub-linear convergence $\mathcal{O}(1/(KR))$ of \texttt{FedSysID} with respect to the number of global iterations $R$ when \texttt{ClientUpdate} follows the \texttt{FedAvg} rule. Additionally, the estimation error of \texttt{FedSysID} converges linearly with complexity $\mathcal{O}(e^{-\beta R})$ when \texttt{ClientUpdate} follows the \texttt{FedLin} rule. To compare these convergence behaviors, we consider an experiment in which $M=50$ clients, with dissimilarity $\epsilon=0.01$, generate $N_i=25$ rollouts of length $T=5$. Figure \ref{fig:FedLin_FedAvg} shows that the estimation error $e_r$ converges faster when \texttt{FedSysID} uses \texttt{FedLin} than when it uses \texttt{FedAvg}. In this experiment, \texttt{FedAvg} requires approximately 60 more global iterations than \texttt{FedLin} to reach the same estimation error.

\begin{figure}[H]
      \centering
     \includegraphics[width=0.5\textwidth]{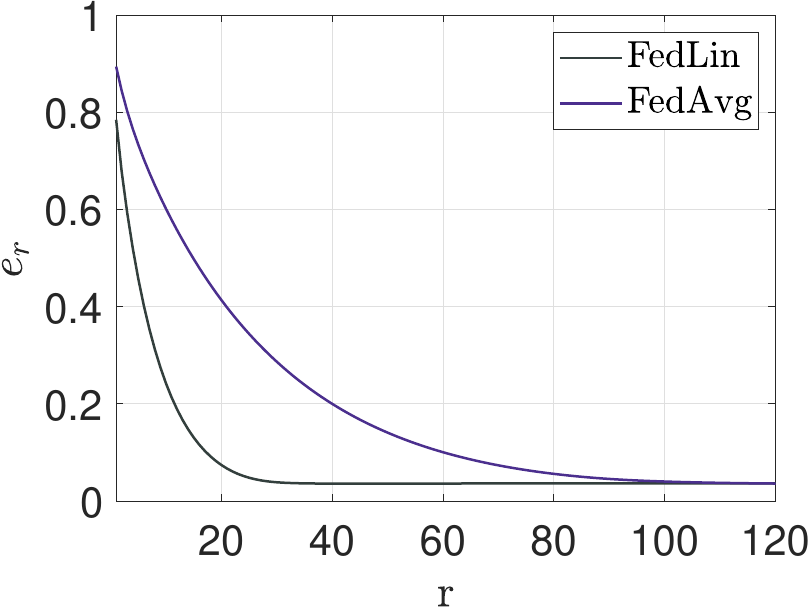}
      \caption{Performance comparison between \texttt{FedAvg} and \texttt{FedLin}, with $M=50$, $N_i=25$, $\epsilon=0.01$ and $K_i=10$.}
      \label{fig:FedLin_FedAvg}
\end{figure}

\makeatletter
\renewcommand*{\@jmlrenddoc}{%
  \phantomsection
  \protected@edef\@currentlabelname{end of \@shorttitle}%
  \label{jmlrend}%
  \global\let\@reprint\@empty
}
\makeatother

\end{document}